%% file: svp.tex
\documentclass[11pt]{article}

\usepackage[bookmarks,colorlinks,breaklinks]{hyperref}  
\hypersetup{linkcolor=blue,citecolor=blue,filecolor=blue,urlcolor=blue} 
\usepackage{fullpage}
\usepackage{amsmath,amsfonts,amsthm,amssymb,xspace,bm}
\usepackage{graphicx}
\usepackage{url,algorithm,algorithmic}

\newcommand{\newref}[2][]{\hyperref[#2]{#1~\ref*{#2}}}
\renewcommand{\eqref}[1]{\hyperref[#1]{(\ref*{#1})}}
  
\theoremstyle{plain}
\newtheorem{theorem}{Theorem}[section]
\newtheorem{lemma}[theorem]{Lemma}

\newtheorem{conjecture}[theorem]{Conjecture}
\newtheorem{definition}[theorem]{Definition}

\theoremstyle{definition}

\DeclareMathOperator*{\pr}{\mathsf{Pr}} 

\DeclareMathOperator*{\ex}{\mathbb{E}}
\newcommand{\rgta}{\rightarrow}
\newcommand{\iprod}[2]{\langle #1, #2 \rangle}   

\newcommand{\rmn}{\mathbb{R}^{m \times n}}
\newcommand{\reals}{\mathbb{R}}

\newcommand{\note}[1]{\marginpar{\tiny *note in TeX*}}
\newcommand{\ignore}[1]{}

\renewcommand{\phi}{\varphi}

\DeclareMathOperator*{\argmin}{argmin}

\newcommand{\fro}[1]{\|#1\|_F}

\author{Raghu Meka \and Prateek Jain \and Inderjit S. Dhillon}
\title{Guaranteed Rank Minimization via Singular Value Projection\footnote{A shorter version of this paper was submitted to NIPS 2009 on June 5, 2009.}}

\begin{document}

\begin{titlepage}

\maketitle
\thispagestyle{empty}
\begin{abstract}
Minimizing the rank of a matrix subject to affine constraints is a
fundamental problem with many important applications in machine
learning and statistics. In this paper we propose a simple and fast
algorithm $\mathsf{SVP}$ (Singular Value Projection) for rank
minimization with affine constraints ($\mathsf{ARMP}$) and show that SVP recovers the minimum rank solution for affine constraints that satisfy the {\sl restricted isometry property}. We show robustness of our method to noise with a strong geometric convergence rate even for noisy measurements. Our results improve upon a recent breakthrough by Recht, Fazel and Parillo \cite{RechtFP2007} and Lee and Bresler \cite{LeeB2009} in three significant ways: 1) our method ($\mathsf{SVP}$) is significantly simpler to analyze and easier to implement, 2) we give recovery guarantees under strictly weaker isometry assumptions 3) we give geometric convergence guarantees for  $\mathsf{SVP}$ and, as demonstrated empirically, $\mathsf{SVP}$ is significantly faster on real-world and synthetic problems. In addition, we address the practically important problem of low-rank matrix completion, which can be seen as a special case of $\mathsf{ARMP}$. However, the affine constraints defining the matrix-completion problem do not obey the {\sl restricted isometry property} in general. We empirically demonstrate that our algorithm recovers low-rank {\sl incoherent} matrices from an almost optimal number of uniformly sampled entries. We make partial progress towards proving exact recovery and provide some intuition for the performance of $\mathsf{SVP}$ applied to matrix completion by showing a more restricted isometry property. Our algorithm outperforms existing methods, such as those of \cite{RechtFP2007,CandesR2008,CandesT2009,CaiCS2008,KeshavanOM2009}, for $\mathsf{ARMP}$ and the matrix-completion problem by an order of magnitude and is also significantly more robust to noise.
\end{abstract}
\end{titlepage}

\input{introduction}
\input{analysis}
\input{matcomp.tex}
\input{weakrip}
\input{related}
\input{results}
\section{Conclusion and Future Work}
\label{sec:conclusions}
There has been a significant amount of work recently in the area of low-rank approximations. Examples include minimizing rank subject to affine constraints, low-rank matrix completion, low-rank plus sparse decomposition. Most of this research, with the exception of Keshavan et al.~\cite{KeshavanOM2009}, relies on relaxing the rank constraint with trace-norm and gives guarantees for recovering the optimal solution under certain additional assumptions. However, trace-norm relaxation based methods are typically hard to analyze and are relatively expensive in practice.

In this paper, we proposed a simple and natural algorithm based on iterative hard-thresholding. We give a simple analysis of our algorithm for the affine rank minimization problem satisfying the restricted isometry property and give geometric convergence guarantees even in the presence of noise. The intermediate steps in our algorithm are less computationally demanding than those of current state-of-the-art methods. We empirically demonstrate that our method is significantly faster and more robust to both uniformly bounded and outlier noise than most existing methods.

An immediate question arising out of our work is to prove our hypothesis bounding the incoherence of the iterates of $\psvd$ for low-rank matrix completion, or otherwise directly prove Conjecture \ref{conj:svp}. Other directions include application of our methods to other problems of similar flavor such as the low-rank plus sparse matrix decomposition \cite{ChandrasekaranSPW2009}, or other matrix completion type problems like minimum dimensionality embedding using partial distance observations \cite{FazelHB2003} and low-rank kernel learning \cite{MekaJCD2008}.
\section*{Acknowledgments}
This research was supported by NSF grant CCF-0431257, NSF grant CCF-0916309 and NSF grant CCF-0728879. We thank the reviewers of NIPS 2009 for their useful comments and for pointing out a mistake in an earlier proof of Theorem \ref{mcrip}.

{\footnotesize
\bibliographystyle{svpurl}
\bibliography{references}
}
\end{document}

%% file: introduction.tex
\newcommand{\avd}{\mathsf{avD}}
\newcommand{\combmc}{\mathsf{COMBMC}}
\newcommand{\po}{\mathcal{P}_{\Omega}}
\newcommand{\eye}{\mathsf{I}}
\newcommand{\aff}{\mathcal{A}}
\newcommand{\affr}{\mathsf{ARMP}}
\newcommand{\raffr}{\mathsf{RARMP}}
\newcommand{\rip}{\mathsf{RIP}}
\newcommand{\ck}{\mathcal{C}(k)}
\newcommand{\pk}{\mathcal{P}_k}
\newcommand{\psvd}{\mathsf{SVP}}
\newcommand{\mcp}{\mathsf{MCP}}
\newcommand{\gks}{\mathsf{GraDeS}}
\newcommand{\up}{U_{\Pi}}

\section{Introduction}
In this paper we study the general affine rank minimization problem (ARMP),
\renewcommand{\theequation}{ARMP}
\begin{equation}\label{mainaffine}
  \min\; rank(X) \;\;\;\; s.t \;\;\;\; \aff(X) = b,\;\;\;\; X \in \rmn,\ b\in \mathbb{R}^d,
\end{equation}
\renewcommand{\theequation}{\arabic{equation}}
\addtocounter{equation}{-1}
where $\aff$  is an affine transformation from $\rmn$ to $\reals^d$. 

The general affine rank minimization problem is of considerable practical interest and many important machine learning problems such as matrix completion, low-dimensional metric embedding, low-rank kernel learning can be viewed as instances of the above problem. Unfortunately, ARMP is NP-hard in general and is also NP-hard to approximate (\cite{MekaJCD2008}). 

Until recently, most known methods for $\affr$ were heuristic in nature with few known rigorous guarantees. The most commonly used heuristic for the problem is to assume a factorization of $X$ and optimize the resulting non-convex problem by alternating minimization \cite{Brand2003,Koren2008,BellK2007}, alternative projections \cite{GrigoriadisB2000} or alternating LMIs \cite{SkeltonIG1997}. Another common approach is to relax the rank constraint to a convex function such as the trace-norm or the log determinant \cite{FazelHB2001}, \cite{FazelHB2003}. However, most of these methods do not have any optimality guarantees. Recently, Meka et al. \cite{MekaJCD2008} proposed online learning based methods for ARMP. However, their methods can only guarantee at best a logarithmic approximation for the minimum rank. 

In a recent breakthrough, Recht et al.~\cite{RechtFP2007} obtained the first nontrivial exact-recovery results for $\affr$ obtaining guaranteed rank minimization for affine transformations $\mathcal{A}$ that satisfy a {\sl restricted isometry property} ($\rip$). Define the isometry constant of $\aff$, $\delta_k$ to be the smallest number such that for all $X \in \rmn$ of rank at most $k$, 
\begin{equation}\label{rip}
 (1- \delta_k) \fro{X}^2 \leq \|\aff(X)\|_2^2 \leq (1+\delta_k) \fro{X}^2. 
\end{equation}

Recht et al.~show that for affine constraints with bounded isometry constants (specifically, $\delta_{5k} < 1/10$), finding the minimum trace-norm solution recovers the minimum rank solution. Their results were later extended to noisy measurements and isometry constants up to $\delta_{3k} < 1/4\sqrt{3}$ by Lee and Bresler \cite{LeeB2009b}. However, even the best existing optimization algorithms for the trace-norm relaxation are relatively inefficient in practice and their results are hard to analyze.

In another recent work, Lee and Bresler \cite{LeeB2009} obtained exact-recovery guarantees for $\affr$ satisfying $\rip$ using a different approach. Lee and Bresler propose an algorithm (ADMiRA) motivated by the {\sl orthogonal matching pursuit} line of work in compressed sensing, and show that for affine constraints with isometry constant $\delta_{4k} \leq 0.04$ their algorithm recovers the optimal solution. They also prove similar guarantees for noisy measurements and provide a geometric convergence rate for their algorithm. However, their method is not very efficient for large datasets and is hard to analyze. 

In this paper we propose a simple and fast algorithm $\psvd$ (Singular Value Projection) based on the classical projected gradient algorithm. We present a simple analysis showing that $\psvd$ recovers the minimum rank solution for affine constraints that satisfy $\rip$ even in the presence of noise and prove the following guarantees. Independent of our work, Goldfarb and Ma \cite{GoldfarbM2009} proposed an algorithm similar to our algorithm. However, their analysis and formulation is different from ours. In particular, their analysis builds on the analysis of Lee and Bresler and they require stronger isometry assumptions, $\delta_{3k} < 1/\sqrt{30}$, than we do. In addition, we make partial progress on analyzing $\psvd$ for the matrix completion problem and proving exact recovery.
\begin{theorem}\label{ripmain}
Suppose the isometry constant of $\aff$ satisfies $\delta_{2k} \leq 1/3$ and let $b = \aff(X^*)$ for a rank-$k$ matrix $X^*$. Then, $\psvd$ (Algorithm~\ref{alg:svp}) with step-size $\eta_t = 1/(1+\delta_{2k})$ converges to $X^*$. Furthermore, $\psvd$ outputs a matrix $X$ of rank at most $k$ such that $\|\aff(X)-b\|_2^2 \leq \epsilon$ in at most $\left\lceil \frac{1}{\log((1-\delta_{2k})/2\delta_{2k})} \log\frac{\|b\|^2}{2\epsilon}\right\rceil$ iterations.
\end{theorem}

\begin{theorem}[Main]\label{ripnoise}
Suppose the isometry constant of $\aff$ satisfies $\delta_{2k} \leq 1/3$ and let $b = \aff(X^*) + e$ for a rank $k$ matrix $X^*$ and an error vector $e \in \reals^d$. Then, $\psvd$ with step-size $\eta_t = 1/(1+\delta_{2k})$ outputs a matrix $X$ of rank at most $k$ such that $\|\aff(X)-b\|_2^2 \leq (C^2+\epsilon)\frac{\|e\|^2}{2}$, $\epsilon\geq 0$, in at most $\left\lceil \frac{1}{\log(1/D)} \log\frac{\|b\|^2}{(C^2+\epsilon)\|e\|^2}\right\rceil$ iterations for universal constants $C,D$.
\end{theorem}

Our analysis of $\psvd$ is motivated by the recent work in the field of compressed sensing by Blumensath and Davies \cite{BlumensathD2009}, Garg and Khandekar \cite{GargK2009}. Our results improve the results of Recht et al.~and Lee and Bresler as follows.
\begin{enumerate}
\item $\psvd$ is considerably simpler to analyze than the methods of Recht et al.~and Lee and Bresler. Further, we need weaker isometry assumptions on $\aff$: we only require $\delta_{2k} < 1/3$ as opposed to $\delta_{5k} < 1/10$ required by Recht et al., $\delta_{3k} < 1/4\sqrt{3}$ required by Lee and Bresler \cite{LeeB2009b} and $\delta_{4k} \leq 0.04$ required by Lee and Bresler \cite{LeeB2009}.
\item $\psvd$ has a strong geometric convergence rate and is faster than using the best trace-norm optimization algorithms and the methods of Lee and Bresler by an order of magnitude.
\end{enumerate}

Although restricted isometry property is natural in settings where the affine constraints contain information about all the entries of the unknown matrix, in several cases of considerable practical interest the affine constraints only contain {\sl local information} and may not satisfy $\rip$ directly. 

One such important problem where $\rip$ does not hold directly is the low-rank matrix completion problem. In the matrix completion problem we are given the entries of an unknown low-rank matrix $X^*$ for ordered pairs $(i,j) \in \Omega \subseteq [m]\times [n]$ and the goal is to complete the missing entries of $X^*$. A highly popular application of the matrix completion problem is in the field of collaborative filtering, where typically the task is to predict user ratings given past ratings of the users. Recently, a lot of attention has been given to the problem  due to the Netflix Challenge \cite{Netflix}. Other applications of matrix completion include triangulation from incomplete data, link prediction in social networks etc.

Similar to $\affr$, the low-rank matrix completion is also NP-hard in general and most methods are heuristic in nature with no theoretical guarantees. The alternating least squares minimization heuristic and its variants \cite{Koren2008,BellK2007} perform the best in practice but are notoriously hard to analyze. 

Recently, Candes and Recht \cite{CandesR2008}, Candes and Tao \cite{CandesT2009} and Keshavan et al.~\cite{KeshavanOM2009} obtained the first non-trivial results for low-rank matrix completion under a few additional assumptions. Broadly, these papers give exact-recovery guarantees when the optimal solution $X^*$ is $\mu$-{\sl incoherent} (see Definition \ref{incoherence}), and the entries $\Omega$ are chosen uniformly at random with $|\Omega| \geq C(\mu,k)\, n\, poly\log n$, where $C(\mu,k)$ depends only on $\mu,k$. However, the algorithms of the above papers, even when using methods tailored specifically for matrix-completion such as those of Cai et al.~\cite{CaiCS2008}, are quite expensive in practice and not very tolerant to noise. 

As low-rank matrix completion is a special case of $\affr$, we can naturally adapt our algorithm $\psvd$ for matrix completion. We demonstrate empirically that for a suitable step-size, $\psvd$ significantly outperforms the methods of \cite{CandesR2008}, \cite{CandesT2009}, \cite{CaiCS2008}, \cite{KeshavanOM2009} in accuracy, computational time and tolerance to noise. Furthermore, our experiments strongly suggest (see Figure~\ref{fig:matcomp1}) that guarantees similar to those of \cite{CandesT2009}, \cite{KeshavanOM2009} hold for $\psvd$, achieving exact recovery for incoherent matrices from an almost optimal number of entries\footnote{It follows from a coupon collector argument that exact-recovery from random samples requires $nk\log n$ samples.}.

Although we do not provide a rigorous proof of exact-recovery for $\psvd$ applied to matrix completion, we make partial progress in this direction and give strong intuition for the performance of $\psvd$. We prove that though the affine constraints defining the matrix-completion problems do not obey the restricted isometry property, they obey the restricted isometry property over incoherent matrices. This weaker $\rip$ condition along with a hypothesis bounding the incoherence of the iterates of $\psvd$ imply exact-recovery of a low-rank incoherent matrix from an almost optimal number of entries. We also provide strong empirical evidence supporting our hypothesis bounding the incoherence of the iterates of $\psvd$ (see Figure \ref{fig:conjecture}).

\begin{figure}[t] 
\begin{center}
      \includegraphics[width=0.45\textwidth, height=4cm]{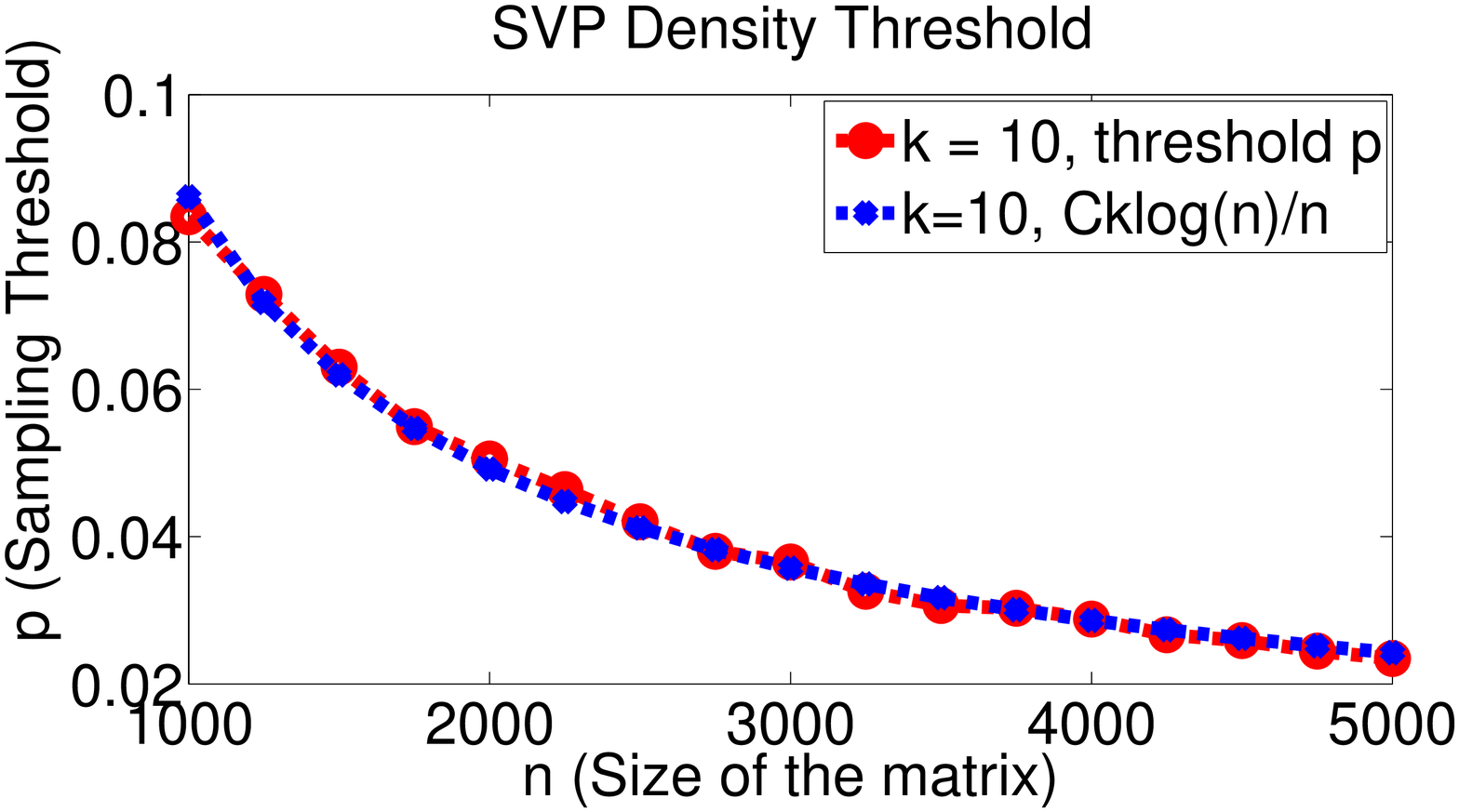}
\caption{Empirical estimate of the sampling density threshold ($p=|\Omega|/mn$) for exact matrix completion using $\mathsf{SVP}$. Note that the threshold scales as $C k \log n/n$ (with $C = 1.28$) almost matching the $k \log n/n$ lowerbound.}
\label{fig:matcomp1}
 \end{center}
\end{figure}

We first present our algorithm $\psvd$ in Section~\ref{sec:svp} and present its analysis for affine constraints satisfying $\rip$ in Section~\ref{sec:analysis}. In Section~\ref{sec:matcomp}, we specialize our algorithm $\psvd$ to the task of low-rank matrix completion and prove a more restricted isometry property for the matrix completion problem. In Section~\ref{sec:results}, we give empirical results for $\mathsf{SVP}$ applied to $\affr$ and matrix-completion on real-world and synthetic problems.

\section{Singular Value Projection (SVP)}
\label{sec:svp}
Consider the following robust formulation of $\affr$ ($\raffr$),
\renewcommand{\theequation}{RARMP}
\begin{equation}\label{robaffine}
  \min_X \;\psi(X) = \frac{1}{2}\|\aff(X) - b\|_2^2 \;\;\; s.t \;\;\; X \in \ck = \{X: rank(X) \leq k\}  .
\end{equation}
\renewcommand{\theequation}{\arabic{equation}}
\addtocounter{equation}{-1}

The hardness of the above problem mainly comes from the non-convexity of the set of low-rank matrices $\ck$. However, in spite of the hardness of the rank constraint, the Euclidean projection onto the non-convex set  $\ck$ can be computed efficiently using singular value decomposition. Our algorithm uses this observation along with the projected gradient method for efficiently minimizing the objective function specified in problem \eqref{robaffine}.

Let $\pk:\rmn \rgta \rmn$ denote the orthogonal projection on to the set $\ck$. That is, $\pk(X) = \argmin_Y \{\fro{Y-X} : Y \in \ck\}$. It is well known that $\pk(X)$ can be computed efficiently by computing the top $k$ singular values and vectors of $X$.

In $\psvd$, a candidate solution to $\affr$ is computed iteratively by starting from the all-zero matrix and adapting the classical projected gradient descent update as follows (Observe that $\nabla \psi(X) = \aff^T(\aff(X) - b)$) :
\begin{equation}\label{mainupdate}
  X^{t+1} \leftarrow \pk\,\left(\, X^t - \eta_t \nabla \psi(X^t)\,\right)=\pk\,\left(\, X^t - \eta_t \aff^T(\aff(X^t) - b)\,\right).
\end{equation}
Algorithm~\ref{alg:svp} presents our $\psvd$ algorithm. Note that the iterates $X^t$ are always low-rank, facilitating faster computation of the SVD. See Section~\ref{sec:compissues} for a more detailed discussion of  the computational issues.
\begin{algorithm}[t]
\caption{Singular Value Projection ($\psvd$) Algorithm}
\label{alg:svp}
  \begin{algorithmic}[1]
    \REQUIRE $\aff, b, \text{tolerance } \varepsilon$, $\eta_t$ for $t=0,1,2,\dots$
    \STATE {\bf Initialize:} $X^0=0$ and $t=0$
    \REPEAT
    \STATE $Y^{t+1}\leftarrow X^t - \eta_t \aff^T(\aff(X^t) - b)$
    \STATE Compute top $k$ singular vectors of $Y^{t+1}$: $U_k$, $\Sigma_k$, $V_k$
    \STATE $X^{t+1}\leftarrow U_k\Sigma_kV_k^T$
    \STATE $t \leftarrow t+1$
    \UNTIL {$\|\aff(X^{t+1})-b\|_2^2\leq \varepsilon$}
  \end{algorithmic}
\end{algorithm}

%% file: analysis.tex
\section{Analysis for Affine Constraints Satisfying $\rip$}
\label{sec:analysis}
We now show that $\psvd$ solves exact rank minimization for affine constraints that satisfy $\rip$ and prove our main results, Theorems \ref{ripmain} and \ref{ripnoise}. 
We first present a lemma that bounds the error at the $(t+1)$-st iteration ($\psi(X^{t+1})$) with respect to the error incurred by the optimal solution ($\psi(X^*)$) and the $t$-th iteration.
\begin{lemma}
\label{lem:lem1}
Let $X^*$ be an optimal solution of \eqref{robaffine} and let $X^t$ be the iterate obtained by SVP algorithm at $t$-th iteration. Then, 
$$\psi(X^{t+1})\leq \psi(X^*)  + \frac{\delta_{2k}}{(1-\delta_{2k})}\|\aff(X^* -X^t)\|_2^2 ,$$
where $\delta_{2k}$ is the rank $2k$ isometry constant of $\aff$. 
\end{lemma}
\begin{proof}
Recall that $\psi(X)=\frac{1}{2}\|\aff(X)-b\|_2^2$. Since $\psi(\cdot)$ is a quadratic function, we have
\begin{align}\label{}
  \psi(X^{t+1}) - \psi(X^t) &=  \iprod{\nabla \psi(X^t)}{X^{t+1}-X^t} + \frac{1}{2} \|\aff(X^{t+1}-X^t)\|_2^2 \nonumber\\
&\leq \iprod{\aff^T(\aff(X^t) - b)}{X^{t+1}-X^t} + \frac{1}{2}\cdot (1+\delta_{2k}) \cdot \fro{X^{t+1}-X^t}^2, \label{pf1eq1}
\end{align}
where inequality \eqref{pf1eq1} follows from $\rip$ applied to the matrix $X^{t+1}-X^t$ of rank at most $2k$. Let $Y^{t+1} = X^t - \frac{1}{1+\delta_{2k}}\aff^T(\aff(X^t)-b)$ and 
\[ f_t(X) = \iprod{\aff^T(\aff(X^t) - b)}{X-X^t} + \frac{1}{2}\cdot (1+\delta_{2k}) \cdot \fro{X -X^t}^2. \]
\\Then,
\begin{align*}\label{}
 f_t(X) &= \frac{1}{2} (1+\delta_{2k})\left[ \fro{X-X^t}^2 + 2 \left\langle \frac{\aff^T(\aff(X^t) - b)}{1+\delta_{2k}}, X-X^t\right\rangle\right]\\
&= \frac{1}{2}(1+\delta_{2k}) \fro{X - Y^{t+1}}^2 - \frac{1}{2(1+\delta_{2k})} \cdot \fro{\aff^T(\aff(X^t)-b)}^2.
\end{align*}

Now, by definition, $\pk(Y^{t+1})=X^{t+1} $ is the minimizer of $f_t(X)$ over all matrices $X \in \ck$ (of rank at most $k$). In particular, $f_t(X^{t+1}) \leq f_t(X^*)$. Thus,
\begin{align}\label{}
  \psi(X^{t+1}) - \psi(X^t) &\leq f_t(X^{t+1}) \leq f_t(X^*)=
 \iprod{\aff^T(\aff(X^t) - b)}{X^*-X^t} + \frac{1}{2} (1+\delta_{2k}) \fro{X^* -X^t}^2 \nonumber\\
&\leq \iprod{\aff^T(\aff(X^t) - b)}{X^*-X^t} + \frac{1}{2}\cdot \frac{1+\delta_{2k}}{1-\delta_{2k}} \|\aff(X^* -X^t)\|_2^2 \label{pf1eq2}\\
&= \psi(X^*) - \psi(X^t) + \frac{\delta_{2k}}{(1-\delta_{2k})}\|\aff(X^* -X^t)\|_2^2  \nonumber,
\end{align}
where inequality \eqref{pf1eq2} follows from $\rip$ applied to $X^*-X^t$.
\end{proof}
We now prove that $\psvd$ obtains the optimal solution for ARMP with restricted isometry property.  
\begin{proof}[Proof of Theorem \ref{ripmain}]
Using Lemma~\ref{lem:lem1} and the fact that $\psi(X^*) = 0$ for the noise-less case, 
\[ \psi(X^{t+1}) \leq \frac{\delta_{2k}}{(1-\delta_{2k})} \|\aff(X^* -X^t)\|_2^2=\frac{2\delta_{2k}}{(1-\delta_{2k})}\psi(X^t).\]
Also, note that for $\delta_{2k}< 1/3$, $\frac{2\delta_{2k}}{(1-\delta_{2k})}< 1$. Hence, $\psi(X^\tau)\leq \epsilon$ where $\tau=\left\lceil \frac{1}{\log((1-\delta_{2k})/2\delta_{2k})} \log\frac{\psi(X^0)}{\epsilon}\right\rceil$. Now, the SVP algorithm is initialized using $X^0=0$, i.e., $\psi(X^0)=\frac{\|b\|^2}{2}$. Hence, $\tau=\left\lceil \frac{1}{\log((1-\delta_{2k})/2\delta_{2k})} \log\frac{\|b\|^2}{2\epsilon}\right\rceil$. 
\end{proof}

Next, we prove the noisy version of Theorem \ref{ripmain}. 
\begin{proof}[Proof of Theorem \ref{ripnoise}]
Let the current solution $X^t$  satisfy $\psi(X^t)\geq C^2\|e\|^2/2$, where $C\geq 0$ is a universal constant. Using Lemma~\ref{lem:lem1} and the fact that $b-\aff(X^*)=e$,
\begin{align*}
\psi(X^{t+1})&\leq \frac{\|e\|_2^2}{2}+\frac{\delta_{2k}}{(1-\delta_{2k})}\|b-\aff(X^t)-e\|_2^2,\\
&\leq \frac{\|e\|_2^2}{2}+\frac{2\delta_{2k}}{(1-\delta_{2k})}\left(\psi(X^t)-e^T(b-\aff(X^t))+\frac{\|e\|^2}{2}\right),\\
&\leq \frac{\psi(X^t)}{C^2}+\frac{2\delta_{2k}}{(1-\delta_{2k})}\left(\psi(X^t)+\frac{2}{C}\psi(X^t)+\frac{1}{C^2}\psi(X^t)\right),\\
&\leq \left(\frac{1}{C^2}+\frac{2\delta_{2k}}{(1-\delta_{2k})}\left(1+\frac{1}{C}\right)^2\right)\psi(X^t)\\
&=D\psi(X^t),
\end{align*}
where $D=\left(\frac{1}{C^2}+\frac{2\delta_{2k}}{(1-\delta_{2k})}\left(1+\frac{1}{C}\right)^2\right)$. Recall that $\delta_{2k}<1/3$. Hence, selecting $C>(1+\delta_{2k})/(1-3\delta_{2k})$, we get $D<1$. Also, $\psi(X^0)=\psi(0)=\|b\|^2/2$. Hence, $\psi(X^\tau)\leq (C^2+\epsilon)\|e\|^2/2$ where $\tau=\left\lceil \frac{1}{\log(1/D)} \log\frac{\|b\|^2}{(C^2+\epsilon)\|e\|^2}\right\rceil$.
\end{proof}

%% file: matcomp.tex
\section{Matrix Completion}\label{sec:matcomp}
We first describe the low-rank matrix completion problem formally. Let $\po:\rmn \rgta \rmn$ denote the projection onto the index set $\Omega$. That is, $(\po(X))_{ij} = X_{ij}$ for $(i,j) \in \Omega$ and $(\po(X))_{ij} = 0$ otherwise. 
Then, the low-rank matrix completion problem ($\mcp$) can be formulated as follows,
\renewcommand{\theequation}{MCP}
\begin{equation}\label{mcrmp}
  \min_X\ \text{rank}(X) \;\;\;\; s.t \;\;\;\; \po(X) = \po(X^*),\ X\in \mathbb{R}^{m\times n}.
\end{equation}
\renewcommand{\theequation}{\arabic{equation}}
\addtocounter{equation}{-1}

Observe that the matrix completion problem is a special case of $\affr$. However, the affine constraints that define $\mcp$, $\po$, do not satisfy $\rip$ in general. Thus Theorems \ref{ripmain}, \ref{ripnoise} above and the results of Recht et al.~\cite{RechtFP2007} do not directly apply to $\mcp$. 
The first non-trivial results for $\mcp$ were obtained recently by Candes and Recht \cite{CandesR2008}, Keshavan et al.~\cite{KeshavanOM2009} and Candes and Tao \cite{CandesT2009}. These works show exact recovery of the unknown matrix $X^*$ when the observed entries are sampled uniformly and $X^*$ is {\sl incoherent} in the sense defined below.

\begin{definition}[Incoherence]\label{incoherence}
A matrix $X \in \rmn$ with singular value decomposition $X = U\Sigma V^T$ is $\mu$-incoherent if 
\[ \max_{i,j} |U_{ij}| \leq \frac{\sqrt{\mu}}{\sqrt{m}},\;\;\; \max_{i,j} |V_{ij}| \leq  \frac{\sqrt{\mu}}{\sqrt{n}} .\]
\end{definition}
Intuitively, high incoherence (i.e., $\mu$ is small) implies that the non-zero entries of $X$ are not concentrated in a small number of entries. Hence, a random sampling of the matrix should provide enough information to reconstruct the entire matrix.

As matrix completion is a special case of $\affr$, we can apply $\psvd$ for matrix completion. We apply $\psvd$ to matrix-completion with step-size $\eta_t = 1/(1+\delta)p$, where $p$ is the density of sampled entries and $0 < \delta < 1/3$ is a parameter depending on how large $p$ is, leading to the update
\begin{equation}\label{mcpupdate}
  X^{t+1} \leftarrow \pk\,\left(\, X^t - \frac{1}{(1+\delta)p} (\po(X^t) - \po(X^*))\,\right).
\end{equation}

We now provide some intuition for our choice of step-size $\eta_t$ and make partial progress towards proving that $\psvd$ achieves exact recovery for low-rank incoherent matrices. We show that though the affine constraints defining $\mcp$, $\po$, do not satisfy $\rip$ for all low-rank matrices, they satisfy $\rip$ for all low-rank incoherent matrices. Thus, if the iterates appearing in $\psvd$ remain incoherent throughout the execution of the algorithm, then Theorem \ref{ripmain} would imply recovery of the unknown entries of the matrix. Empirical evidence strongly supports our hypothesis that the incoherence of the iterates arising in $\psvd$ remains bounded.

Figure~\ref{fig:matcomp1} plots the threshold sampling density $p$ beyond which matrix completion for randomly generated matrices is solved exactly by $\psvd$ for fixed $k$ and varying matrix sizes $n$. Note that the density threshold matches the optimal bound of $O(k\log n/n)$ with the constant being $C=1.28$. Figure \ref{fig:conjecture} plots the maximum incoherence $\max_t \mu(X^t) = \sqrt{n}\,\max_{t,i,j}|U^t_{ij}|$, where $U^t$ are the left singular vectors of the intermediate iterates $X^t$ computed by $\psvd$. The figure clearly shows that the incoherence $\mu(X^t)$ of the iterates is bounded by a constant independent of the matrix size $n$ and density $p$ throughout the execution of $\psvd$.

 Fix an incoherent matrix $X\in\mathbb{R}^{m\times n}$ of rank at most $k$ and let $\Omega$ be sampled according to the {\sl Bernoulli model} with each $(i,j) \in \Omega$ independently with probability $p$. Then, $E[\fro{\po(X)}^2] = p \fro{X}^2$. Further, by Chernoff bounds, for $\delta > 0$, $p \geq C k^2\log n/m$ for a universal constant $C$, with high probability 
\begin{equation}\label{porip}
 (1 - \delta) p\, \fro{X}^2 \;\leq\; \fro{\po(X)}^2 \; \leq\; (1+\delta) p\, \fro{X}^2 .  
\end{equation}

Combining the above Chernoff bound estimate with a union bound over low-rank incoherent matrices, we obtain the following restricted isometry property for the projection operator $\po$ restricted to low-rank incoherent matrices. See Section \ref{sec:weakrip} for a detailed proof.
\begin{theorem}\label{mcrip}
There exists a constant $C\geq 0$ such that the following holds for all $0 < \delta < 1$, $\mu \geq 1$, $n\geq m\geq 3$: For $\Omega \subseteq [m] \times [n]$ chosen according to the Bernoulli model with density $p \geq C \mu^2 k^2 \log n/\delta^2 m$, with probability at least $1-\exp(-n \log n)$, the restricted isometry property in \eqref{porip} holds for all $\mu$-incoherent matrices $X$ of rank at most $k$. 
\end{theorem}

Motivated by the above theorem and supported by empirical evidence (Figures \ref{fig:matcomp1}, \ref{fig:conjecture}) we hypothesize that $\psvd$ achieves exact recovery from an almost optimal number of samples.

\begin{conjecture}
\label{conj:svp}
Fix $\mu,k$ and $\delta \leq 1/3$. Then, there exists a constant $C$ such that for a $\mu$-incoherent matrix $X^*$ of rank at most $k$ and $\Omega$ sampled from the Bernoulli model with density $p \geq C \mu^2 k^2 \log n/\delta^2 m$, $\psvd$ with step-size $\eta_t = 1/(1+\delta)p$ converges to $X^*$ with high probability. Moreover, $\psvd$ outputs a matrix $X$ of rank at most $k$ such that $\fro{\po(X)- \po(X^*)}^2 \leq \epsilon$ after $O_{\mu,k}\left(\left\lceil \log \left(\frac{1}{\epsilon} \right) \right\rceil\right)$ iterations.
\end{conjecture}

\begin{figure} 
\begin{center}
      \includegraphics[width=0.45\textwidth, height=4cm]{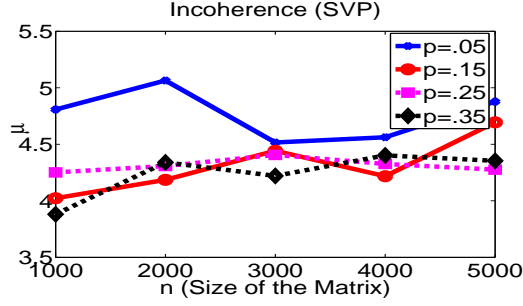}
\caption {Maximum incoherence $\max_t \mu(X^t)$  over the iterates of $\psvd$ for varying densities $p$ and sizes $n$ of randomly generated matrices (averaged over 20 runs). Note that the incoherence is bounded by a constant, supporting Conjecture~\ref{conj:svp}.}
\label{fig:conjecture}
 \end{center}
\end{figure}


%% file: weakrip.tex
\subsection{$\rip$ for Matrix Completion on Incoherent Matrices}\label{sec:weakrip} 
We now prove the $\rip$ property of Theorem \ref{mcrip} for the projection operator $\po$. To prove Theorem \ref{mcrip} we first show the theorem for a {\sl discrete} collection of matrices using Chernoff type large-deviation bounds and use standard quantization arguments to generalize to the continuous case. We first introduce some notation.

\newcommand{\mnorm}[1]{\|{#1}\|_\mathtt{mx}}
\begin{definition}
  For a matrix $X \in \rmn$, let $\mnorm{X} = \max_{i,j} |X_{ij}|$ and call $X$ $\alpha$-regular if 
\[ \mnorm{X} \leq \frac{\alpha}{\sqrt{mn}}\cdot \fro{X}.\]
\end{definition}

We need  Bernstein's inequality \cite{Bernstein2009} stated below. 
\begin{lemma}[{\bf Bernstein's inequality}]
\label{lem:bernstein}
Let $X_1, X_2, \dots, X_n$ be independent random variables with $E[X_i]=0, \forall i$. Furthermore, let $|X_i|\leq M$. Then,
$$P[\sum_i X_i >t]\leq \exp\left(-\frac{t^2/2}{\sum_i Var(X_i)+Mt/3}\right).$$
\end{lemma}
 
\begin{lemma}\label{ripmc1}
Fix an $\alpha$-regular $X \in \rmn$ and $0 < \delta < 1$. Then, for $\Omega \subseteq [m] \times [n]$ chosen according to the Bernoulli model, with each pair $(i,j) \in \Omega$ chosen independently with probability $p$,
\[ \pr[\,\left|\fro{\po(X)}^2 - p \fro{X}^2\right| \geq \delta p \fro{X}^2\,] \leq 2 \exp\left(-\frac{\delta^2 p mn}{3\, \alpha^2}\right).\]
\end{lemma}
\begin{proof}
For $(i,j) \in [m] \times [n]$, let $\omega_{ij}$ be the indicator variables with $\omega_{ij} = 1$ if $(i,j) \in \Omega$ and $0$ otherwise. Then, $\omega_{ij}$ are independent random variables with $\pr[\omega_{ij} = 1] = p$. Let random variable $Z_{ij} = \omega_{ij} X_{ij}^2$. Note that, 
$$E[Z_{ij}]=pX_{ij}^2,\ Var(Z_{ij})=p(1-p)X_{ij}^4.$$
Observe that $|Z_{ij} - E[Z_{ij}]| \leq |X_{ij}|^2 \leq (\alpha^2/mn) \cdot \fro{X}^2$. Thus, 
\begin{equation}
\label{eq:boundM}
M = \max_{i,j} |Z_{ij} - \ex[Z_{ij}]| \leq \frac{\alpha^2}{mn} \fro{X}^2.
\end{equation}
Now, define random variable $S=\sum_{i,j} Z_{ij}=\sum_{i,j} \omega_{ij} X_{ij}^2=\fro{\po(X)}^2$. Note that, $E[S]=p\fro{X}^2$. Since, $Z_{ij}$ are independent random variables,
\begin{equation}
\label{eq:boundVar}
 Var(S) = \sum_{i,j} p(1-p)\, X_{ij}^4 \leq p \, (\max_{i,j} X_{ij}^2) \cdot \sum_{i,j} X_{ij}^2 \leq \frac{p\,\alpha^2}{mn} \fro{X}^4.
\end{equation}

Using Bernstein's inequality (Lemma~\ref{lem:bernstein}) for $S$ with $t = \delta p \fro{X}^2$ and Equations \eqref{eq:boundM} and \eqref{eq:boundVar} we get,
\begin{align*}
  \pr[|S - \ex[S]| > t ] &\leq 2 \exp\left(\frac{-t^2/2}{Var(Z) + Mt/3}\right)\\
    &\leq 2 \exp\left(-\frac{\delta^2 p mn}{\alpha^2(1+\delta/3)}\right)\\
& \leq 2 \exp\left(-\frac{\delta^2 p mn}{3 \alpha^2}\right).
\end{align*}

\end{proof}

We now discretize the space of low-rank incoherent matrices so as to be able to use the above lemma with a union bound. We need the following simple lemmas.

\begin{lemma}\label{mxinc}
Let $X \in \rmn$ be a $\mu$-incoherent matrix of rank at most $k$. Then $X$ is $\mu \sqrt{k}$-regular.
\end{lemma}
\begin{proof}
Let $X = U \Sigma V^T$ be the singular value decomposition of $X$. Then, $X_{ij} = U_i \Sigma V_j^T$, where $U_i,V_j$ are the $i$'th and $j$'th rows of $U,V$ respectively. Now,
\begin{equation*}\label{}
  |X_{ij}|=|e_i^TU \Sigma V^Te_j| = |\sum_{l=1}^k U_{il} \Sigma_{ll} V_{jl}|  \leq \sum_{l=1}^k \Sigma_{ll} |U_{il}| |V_{jl}| .
\end{equation*}
Since $X$ is $\mu$-incoherent, 
\begin{equation*}
|X_{ij}|\leq \sum_{l=1}^k \Sigma_{ll} |U_{il}| |V_{jl}|\leq \frac{\mu}{\sqrt{mn}} \cdot (\sum_{l=1}^k \Sigma_{ll})\leq\frac{\mu}{\sqrt{mn}} \cdot \sqrt{k}\cdot (\sum_{l=1}^k \Sigma_{ll}^2)^{1/2}= \frac{\mu\sqrt{k}}{\sqrt{mn}} \cdot \fro{X}.
\end{equation*}
\end{proof}

\begin{lemma}\label{triva}
Let $a,b,c,x,y,z \in [-1,1]$. Then, 
\[ |abc - xyz| \leq |a-x| + |b-y| + |c-z|.\]
\end{lemma}
The following lemma shows that the space of low-rank $\mu$-incoherent matrices can be discretized into a reasonably small set of regular matrices such that every low-rank  $\mu$-incoherent matrix is close to a matrix from the set. 
\begin{lemma}\label{epsnetinc}
For all $0 < \epsilon < 1/2$, $\mu \geq 1$, $m,n\geq 3$ and $k\geq 1$, there exists a set $S(\mu,\epsilon) \subseteq \rmn$ with  $|S(\mu,\epsilon)| \leq (mnk/\epsilon)^{3\,(m+n)k}$ such that the following holds. For any $\mu$-incoherent $X \in \rmn$ of rank $k$ with $\|X\|_2 = 1$, there exists $Y \in S(\mu,\epsilon)$ such that $\fro{Y-X} < \epsilon$ and $Y$ is $(4\mu\sqrt{k})$-regular.
\end{lemma}
\begin{proof}
We construct $S(\mu,\epsilon)$ by discretizing the space of low-rank incoherent matrices. Let $\rho = \epsilon/\sqrt{9k^2 mn}$ and $D(\rho) = \{ \rho \, i: i \in \mathbb{Z}, |i| < \lfloor 1/\rho\rfloor\}$. Let 
\[ U(\rho) = \{U \in \reals^{m \times k}: U_{ij} \in (\sqrt{\mu/m}) \cdot D(\rho)\; \},\] 
\[ V(\rho) = \{V \in \reals^{n \times k}: V_{ij} \in (\sqrt{\mu/n}) \cdot D(\rho)\; \}, \]
\[ \Sigma(\rho) = \{\Sigma \in \reals^{k \times k}: \Sigma_{ij} = 0, i \neq j,\; \Sigma_{ii} \in D(\rho) \},\]
\[ S(\mu,\epsilon) = \{\,U \Sigma V^T: U \in U(\rho), \Sigma \in \Sigma(\rho), V \in V(\rho)\,\}.\]
We will show that $S(\mu,\epsilon)$ satisfies the conditions of the Lemma. Observe that $|D(\rho)| < 2/\rho$. Thus, 
\[|U(\rho)| < (2/\rho)^{mk},\;\;\; |V(\rho)| < (2/\rho)^{nk},\;\;\; |\Sigma(\rho)| < (2/\rho)^{k}.\]
Hence, $|S(\mu,\epsilon)| < (2/\rho)^{mk+nk+k} < (mnk/\epsilon)^{3 (m+n) k}$.

Fix a $\mu$-incoherent $X \in \rmn$ of rank at most $k$ with $\|X\|_2 = 1$. Let the singular value decomposition of $X$ be $X = U \Sigma V^T$.
Let $U_1$ be the matrix obtained by rounding entries of $U$ to integer multiples of $\sqrt{\mu}\,\rho/\sqrt{m}$ as follows: for $(i,l) \in [m] \times [k]$, let
\[ (U_1)_{il} =  \frac{\sqrt{\mu}\rho}{\sqrt{m}} \cdot \left\lfloor U_{il} \frac{\sqrt{m}}{\sqrt{\mu}\,\rho}\right \rfloor .\]
Now, since $|U_{il}| \leq \sqrt{\mu}/\sqrt{m}$, it follows that $U_1 \in U(\rho)$. Further, for all $i \in [m], l \in [k]$, 
\[ |(U_1)_{il} - U_{il}| < \frac{\sqrt{\mu}}{\sqrt{m}} \, \rho \leq \rho.\]
Similarly, define $V_1,\Sigma_1$ by rounding entries of $V,\Sigma$ to integer multiples of $\sqrt{\mu}\,\rho/\sqrt{n}$ and $\rho$ respectively. Then, $V_1 \in V(\rho)$, $\Sigma_1 \in \Sigma(\rho)$ and for $(j,l) \in [n] \times [k]$,    
\[ |(V_1)_{jl} - V_{jl}| < \frac{\sqrt{\mu}\rho}{\sqrt{n}} \leq \rho,\;\;\; |(\Sigma_1)_{ll} - \Sigma_{ll}| < \rho.\]
Let $X(\rho) = U_1 \Sigma_1 V_1^T$. Then, by the above equations and Lemma~\ref{triva}, for $i\in [m], l \in [k], j \in [n]$, 
\begin{equation*}\label{}
  |(U_1)_{il} (\Sigma_1)_{ll} (V_1)_{jl}- U_{il} \Sigma_{ll} V_{jl}| < 3 \rho.
\end{equation*}
Thus, for $i,j \in [m] \times [n]$, 
\begin{align}\label{eq1}
  |X(\rho)_{ij}-X_{ij}| &= |\sum_{l=1}^k (U_1)_{il} (\Sigma_1)_{ll} (V_1)_{jl} - U_{il} \Sigma_{ll} V_{jl}|\nonumber \\
  &\leq \sum_{l=1}^k | (U_1)_{il} (\Sigma_1)_{ll} (V_1)_{jl} - U_{il} \Sigma_{ll} V_{jl}| \nonumber\\
  &< 3 k \rho.
\end{align}
Using Lemma~\ref{mxinc} and Equation \eqref{eq1}
\[ \mnorm{X(\rho)} < \mnorm{X} + 3 k \rho \leq \frac{\mu\sqrt{k}}{\sqrt{mn}} \cdot \fro{X} + \frac{\epsilon}{\sqrt{mn}}.\]
Also, using \eqref{eq1},
\[ \fro{X(\rho)-X}^2 = \sum_{i,j} |X(\rho)_{ij}-X_{ij}|^2 < 9 k^2 mn \rho^2 = \epsilon^2.\]
Furthermore, using triangular inequality, $\fro{X(\rho)} > \fro{X} - \epsilon > \fro{X}/2$. Since, $\epsilon< 1$ and $\mu\sqrt{k}\|X\|_F\geq 1$,
\[\mnorm{X(\rho)} <  \frac{2 \mu\sqrt{k}}{\sqrt{mn}} \cdot \fro{X} < \frac{4\mu\sqrt{k}}{\sqrt{mn}} \cdot \fro{X(\rho)} .\]
Thus, $X(\rho)$ is $4\mu\sqrt{k}$-regular. The lemma now follows by taking $Y = X(\rho)$.
\end{proof}

We now prove Theorem \ref{mcrip} by combining Lemmas \ref{ripmc1} and \ref{epsnetinc}. 
\begin{proof}[Proof of Theorem \ref{mcrip}]
Let $m\leq n$, $\epsilon = \delta /9 mn k$ and 
\[ S'(\mu,\epsilon) = \{Y: Y \in S(\mu,\epsilon), Y\text{ is $4\mu\sqrt{k}$-regular}\},\]
where $S(\mu,\epsilon)$ is as in Lemma \ref{epsnetinc}. Then, by Lemma \ref{mcrip} and union bound, 
\begin{align*}  
 \pr\left[\,\left|\fro{\po(Y)}^2 - p \fro{Y }^2\right| \geq \delta p \fro{Y}^2\,\text{ for some $Y \in S'(\mu,\epsilon)$}\,\right]  &\leq 2 \left(\frac{mnk}{\epsilon}\right)^{3 (m+n) k} \exp\left(\frac{-\delta^2 pmn}{16\mu^2 k}\right)\\
&\leq \exp(C_1n k \log n) \cdot \exp\left(\frac{-\delta^2 pmn}{16\mu^2 k}\right),
\end{align*}
where $C_1\geq 0$ is a constant independent of $m,n,k$. 

Thus,  if $p > C \mu^2 k^2\log n/\delta^2 m$, where $C=16(C_1+1)$, with probability at least $1-\exp(- n\log n)$, the following holds 
\begin{equation}
  \label{eq:3}
\forall Y \in S'(\mu,\epsilon),\;\;\; |\fro{\po(Y)}^2 - p \fro{Y }^2| \leq \delta p \fro{Y}^2.  
\end{equation}
  As the statement of the theorem is invariant under scaling, it is enough to show the statement for all $\mu$-incoherent matrices $X$ of rank at most $k$ and $\|X\|_2 = 1$. Fix such a $X$ and suppose that \eqref{eq:3} holds. Now, by Lemma \ref{epsnetinc} there exists $Y \in S'(\mu,\epsilon)$ such that $\fro{Y-X} \leq \epsilon$. Moreover,
\[ \fro{Y}^2 \leq (\fro{X} + \epsilon)^2 \leq \fro{X}^2 + 2 \epsilon \fro{X} + \epsilon^2 \leq \fro{X}^2 + 3 \epsilon k.\]
Proceeding similarly, we can show that 
\begin{equation}
  \label{eq:1}
   |\fro{X}^2 - \fro{Y}^2| \leq 3 \epsilon k.
\end{equation}
Further, starting with $\fro{\po(Y-X)} \leq \fro{Y-X} \leq \epsilon$ and arguing as above we get that 
\begin{equation}
  \label{eq:2}
 |\fro{\po(Y)}^2 - \fro{\po(X)}^2| \leq 3 \epsilon k.  
\end{equation}
Combining inequalities \eqref{eq:1}, \eqref{eq:2} above, we have
\begin{align*}
  |\fro{\po(X)}^2 - p\fro{X}^2| &\leq |\fro{\po(X)}^2 - \fro{\po(Y)}^2|  + p\,|\fro{X}^2 - \fro{Y}^2| + | \fro{\po(Y)}^2 - p \fro{Y}^2|\\
  &\leq 6 \epsilon k + \delta p \fro{Y}^2 \text{\hspace{1.48in} from \eqref{eq:3}, \eqref{eq:1}, \eqref{eq:2}}\\
  &\leq 6 \epsilon k + \delta p (\fro{X}^2 + 3 \epsilon k) \text{\hspace{1in} from \eqref{eq:1}}\\
  &\leq 9 \epsilon k + \delta p \fro{X}^2\\
  &\leq 2 \delta p \fro{X}^2. \text{\hspace{1.75in} Since } \fro{X}^2\geq 1
\end{align*}
The theorem now follows.
\end{proof}


%% file: related.tex
\section{Computational Issues and Related Work}\label{sec:compissues}
The affine rank minimization problem is a natural generalization to matrices of the following compressed sensing problem for vectors:
\begin{align}
  \min_x &\ \|x\|_0,\nonumber\\
\text{s.t.}&\ Ax=b,
\label{eq:cs}
\end{align}
where $\|x\|_0$ is the $l_0$ norm (size of the support) of $x\in \mathbb{R}^n$, $A\in \mathbb{R}^{m\times n}$ is the sensing matrix and $b\in\mathbb{R}^m$ are the measurements. 
Just as in the case of $\affr$, the compressed sensing problem is also NP-hard in general. 

However, a number of methods have been proposed recently to solve the problem for restricted families of sensing matrices. Most of the methods with provable theoretical guarantees assume that the sensing matrix $A$ satisfies restricted isometry properties similar to those in \eqref{rip}. Broadly speaking, existing compressed sensing approaches can be divided into three categories:
\begin{itemize}
\item {\bf $l_1$ relaxation}: These methods relax the non-convex $l_0$ objective function to the convex $l_1$ objective function \cite{CandesT2005,CandesR2007,Fuchs2005,DonohoET2006}. At a high level these results show that if the sensing matrix $A$ obeys $\rip$ or other $\rip$ like properties, then $l_1$ relaxation recovers the optimal sparse solution from an almost optimal $O(k \log n)$ measurements. 
\item {\bf Basis pursuit}: These methods greedily search for the subset of columns of $A$ that would span the optimal solution. Specifically, in each iteration, columns of the sensing matrix that have the highest correlation with the current {\sl residual} measurement vector are greedily added to the {\sl basis}. Assuming $\rip$, basis pursuit methods also guarantee recovery of the optimal solution from a near optimal number of measurements \cite{TroppN2008,NeedellTV2008}. 
\item {\bf Iterative Hard Thresholding (IHT)}: IHT based methods try to minimize $l_0$ norm directly by hard thresholding \cite{BlumensathD2009,GargK2009} the current candidate solution to a small support vector. Here again, exact-recovery guarantees are known assuming $\rip$. Recently, Garg and Khandekar \cite{GargK2009} demonstrated that their GradeS method  outperforms most of the existing compressed sensing algorithms empirically. 
\end{itemize}

As $\affr$ is a generalization of problem \eqref{eq:cs}, it is natural to ask if the above compressed sensing algorithms can be generalized to solve $\affr$. Interestingly, the answer is yes. Trace-norm relaxation approaches \cite{RechtFP2007} can be seen as a direct generalization of the $l_1$ relaxation approach. Similarly, the ADMiRA algorithm of Lee and Bresler \cite{LeeB2009} generalizes the CoSAMP algorithm of Tropp and Needell \cite{TroppN2008}. Finally, our approach is a generalization of the IHT approach.  Table~\ref{tab:armp} summarizes these three approaches and compares them in terms of a few desirable characteristics an algorithm for $\affr$ should have. 

 \begin{table}[ht]
   \centering
   \begin{tabular}[ht]{|c|c|c|c|c|}
 \hline
 Method&Generalization of&$\rip$ constant&Rate of Convergence&Noisy Measurements\\\hline
 Trace-norm  \cite{RechtFP2007}& $l_1$ relaxation&$\delta_{5k}< 1/10$&Not known&No\\\hline
 Trace-norm \cite{LeeB2009b} & $l_1$ relaxation &$\delta_{3k} < 1/4\sqrt{3}$ & Not known & Yes \\\hline
 ADMiRA \cite{LeeB2009}& Basis Pursuit&$\delta_{4k} < 1/\sqrt{32}$&Geometric&Yes\\\hline
 SVP, this paper& IHT& $\delta_{2k}\leq 1/3$ & Geometric & Yes\\\hline
   \end{tabular}
   \caption{Comparison of the existing approaches for $\affr$ with our $\psvd$ method}
   \label{tab:armp}
 \end{table}

Minimizing the trace-norm of a matrix subject to affine constraints can be cast as a semi-definite programming problem. However, algorithms for semi-definite programming, as used by most methods for minimizing trace-norm, are prohibitively expensive even for moderately large datasets. Recently, a variety of methods mostly based on iterative soft-thresholding have been proposed to solve the trace-norm minimization problem efficiently. For instance, Cai et al.~\cite{CaiCS2008} proposed a Singular Value Thresholding (SVT) algorithm which is based on Uzawa's algorithm\cite{ArrowHU1958}. A related approach based on linearized Bregman iteration was proposed by Ma et al.~\cite{MaGC2009}. Toh and Yun \cite{TohY2009}, while Ji and Ye \cite{JiY2009} proposed Nesterov's projected gradient based methods for optimizing the trace-norm. 

While the soft-thresholding based methods for trace-norm minimization are significantly faster than semi-definite programming approaches they suffer from an important bottleneck: though the final solution to the trace-norm minimization is a low-rank matrix, the rank of the iterate in intermediate iterations can be large. In contrast, the rank of the iterates in our method is always equal to the rank of the optimal solution.  

Also, though minimizing the trace-norm approximates the low-rank solution even in the presence of noise (see \cite{CandesP2009}, \cite{LeeB2009b} for instance), noise poses considerable computational challenges for trace-norm optimization. Cai et al.~propose a variant of SVT for handling noise that performs moderately well for uniformly bounded noise. However, the performance of SVT worsens considerably in the presence of outlier noise. $\psvd$ on the other hand is robust to both outlier and uniformly bounded noise as it minimizes the cumulative loss function $\|\aff(X)-b\|_2^2$.

For the case of low-rank matrix completion, Candes and Recht \cite{CandesR2008} obtained the first non-trivial results for the problem obtaining guaranteed completion for {\sl incoherent matrices} $X^*$ and randomly sampled entries $\Omega$. Candes and Recht show that for $X^*$ $\mu$-incoherent and $\Omega$ chosen at random with $|\Omega| \geq C(\mu)\,k^2 n^{1.2}$, trace-norm relaxation recovers the optimal solution. Building on the work of Candes and Recht, Candes and Tao \cite{CandesT2009} obtained the near-optimal bound of $|\Omega| \geq \min(C \mu^4 k^2 n \log^2 n, C \mu^2 k n \log^6 n)$ for exact-recovery via trace-norm minimization. However, the analysis of Candes and Recht, Candes and Tao is considerably complicated and minimizing trace-norm, even when using methods tailored for matrix-completion such as those of Cai et al.~is relatively expensive in practice. 

For the case of matrix completion, SVT has the important property that the intermediate iterations of the algorithm only require computing the singular value decomposition of a sparse matrix. This facilitates the use of fast SVD computing package such as PROPACK \cite{Larsen} that only require subroutines that compute matrix-vector products. 

Our $\psvd$ algorithm has a similar property facilitating fast computation of the update in equation \eqref{mcpupdate}; each iteration of $\psvd$ involves computing the SVD of the matrix $Y = X^t + \po(X^t-X^*)$, where $X^t$ is a matrix of rank at most $k$ whose SVD we know and $\po(X^t-X^*)$ is a sparse matrix. Thus, we can compute matrix-vector products of the form $Y x$ in time $O((m+n)k + |\Omega|)$.

In a different line of work, Keshavan et al.~\cite{KeshavanOM2009} obtained exact-recovery from uniformly sampled $\Omega$ with $|\Omega| \geq C(\mu,k) \, n \log n$ using different techniques. The first iteration of $\psvd$ is similar to the first step of Keshavan et al. However, after the first iteration, Keshavan et al.~use a sophisticated alternating minimization algorithm based on gradient descent on the Grassmannian manifold of low-rank matrices.  However, convergence of their alternating minimization algorithm is slow. The simplicity of the updates in $\psvd$ makes it both easier to implement and significantly less computationally intensive than the alternating minimization algorithm of Keshavan et al.

A related problem to the matrix completion problem is the problem of {\sl low-rank plus sparse} decomposition of a matrix addressed by Chandrasekaran et al.~\cite{ChandrasekaranSPW2009} and Wright et al.~\cite{WrightGRM2009}. Interestingly, Wright et al.~\cite{WrightGRM2009} show that the low-rank matrix completion problem can be reduced to the low-rank plus sparse decomposition problem. Here again, their method relies on the trace-norm relaxation and is significantly more computationally intensive than our algorithm. 
\subsection{Selecting rank ($k$)}
A drawback of our SVP method it requires rank $k$ of the optimal solution to be known beforehand. For ARMP, we propose using the following heuristic: run SVP with some initial guess $k$ and increment it by a fixed number (e.g, 10) until error $\|{\mathcal A}X-b\|^2\|$ incurred by SVP doesn't change.

For the matrix completion problem, in the first step of our SVP method, we compute singular values incrementally till we find a significant gap between singular values. Our heuristic is justified because: Keshavan et al. \cite{KeshavanOM2009} show that the top $k$ ($k$ being rank of optimal solution) singular values of the sampled matrix approximate the underlying matrix well, i.e., there should be a gap between $k$-th and $k+1$-th singular value.  

%% file: results.tex
\section{Experimental Results}
\label{sec:results}
In this section, we empirically evaluate our $\psvd$ method for the affine rank minimization and low-rank matrix completion problems. For both problems we present empirical results on synthetic as well as real-world datasets. For $\affr$ we compare our method against the trace-norm based singular value thresholding (SVT) method \cite{CaiCS2008}. Note that although Cai et al.~present the SVT algorithm in the context of matrix completion problem, it can be easily adapted for $\affr$. For matrix completion we compare against SVT, ADMiRA \cite{LeeB2009}, the spectral matrix completion (SMC) method of Keshavan et al.~\cite{KeshavanOM2009}, and regularized alternating least squares minimization (ALS). We use our own implementation of ALS and SVT for ARMP, while for matrix completion we use the code provided by the respective authors for SVT, ADMiRA and SMC. We report results averaged over $20$ runs. All the methods are implemented in Matlab and use mex files. 

\subsection{Affine Rank Minimization}
We first compare our method against SVT on random instances of $\affr$. We generate  random matrices $X\in\mathbb{R}^{n\times n}$ of different sizes $n$ and fixed rank $k=5$. We then generate $d=6kn$  random affine constraint matrices $A_i, 1\leq i\leq d$ and compute $b={\cal A}(X)$. Figure~\ref{fig:armp} (a) compares the computational time required by $\psvd$ and SVT (in $\log$-scale) for achieving a relative error ($\|{\cal A}(X)-\bm{b}\|_2/\|\bm{b}\|_2$) of $10^{-3}$, and shows that our method requires many fewer iterations and is significantly faster than SVT. 

Next we evaluate our method for the problem of matrix reconstruction from random measurements. As in Recht et al. \cite{RechtFP2007}, we use the MIT logo as the test image for reconstruction. the MIT logo we use is a $38\times 73$ image and has rank four. For reconstruction, we generate random measurement matrices $A_i$ and measure $b_i=Tr(A_iX)$. Figure~\ref{fig:armp} (b) shows that our method incurs significantly smaller reconstruction error than SVT with lower number of iterations. 

\begin{figure*} [ht!]
\begin{center}
    \begin{tabular}[c]{@{\hspace{-0.02\textwidth}}c@{\hspace{-0.02\textwidth}}c@{\hspace{-0.01\textwidth}}c@{}}
      \includegraphics[width=0.45\textwidth, height=4cm]{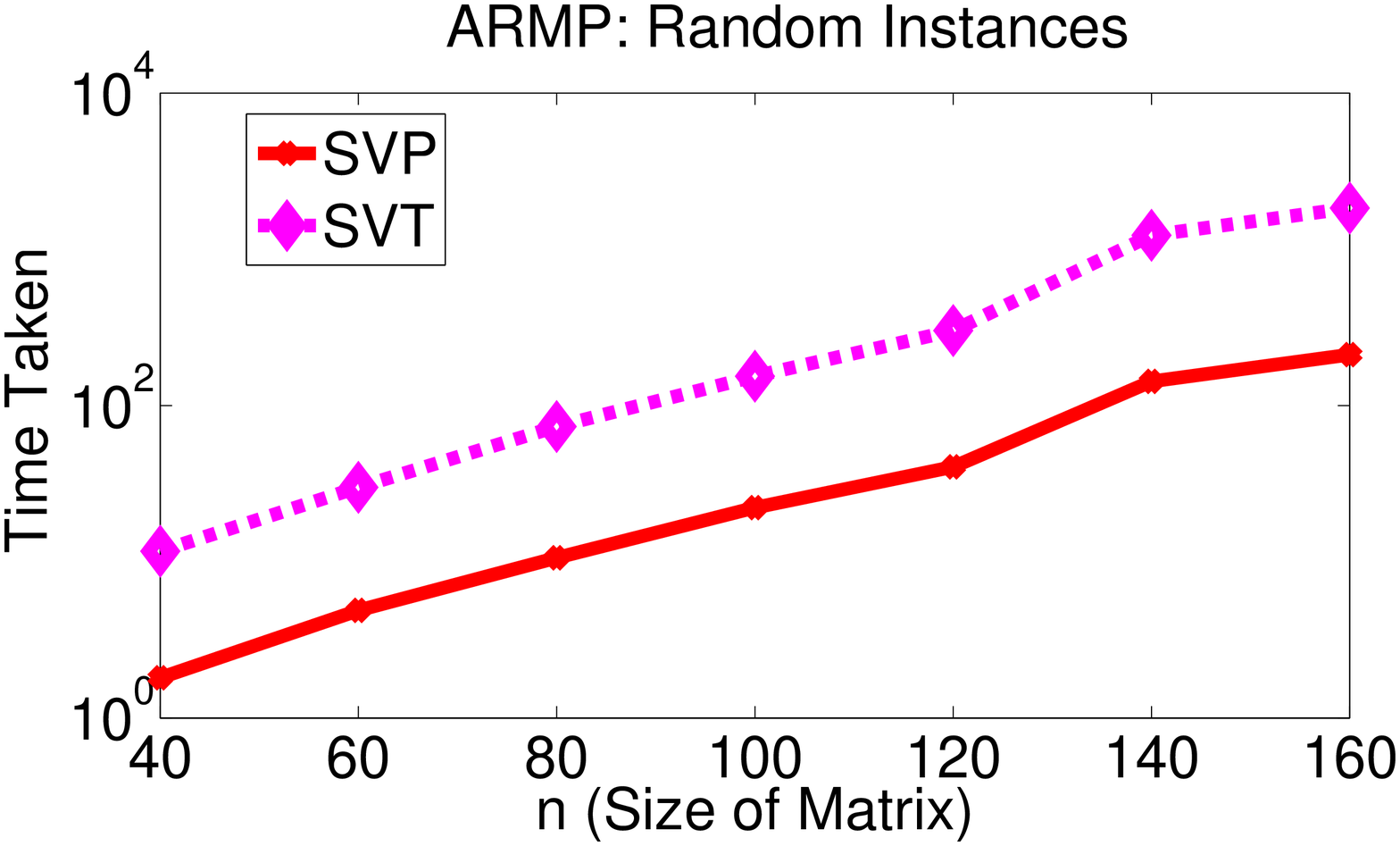}&
      \includegraphics[width=0.45\textwidth, height=4cm]{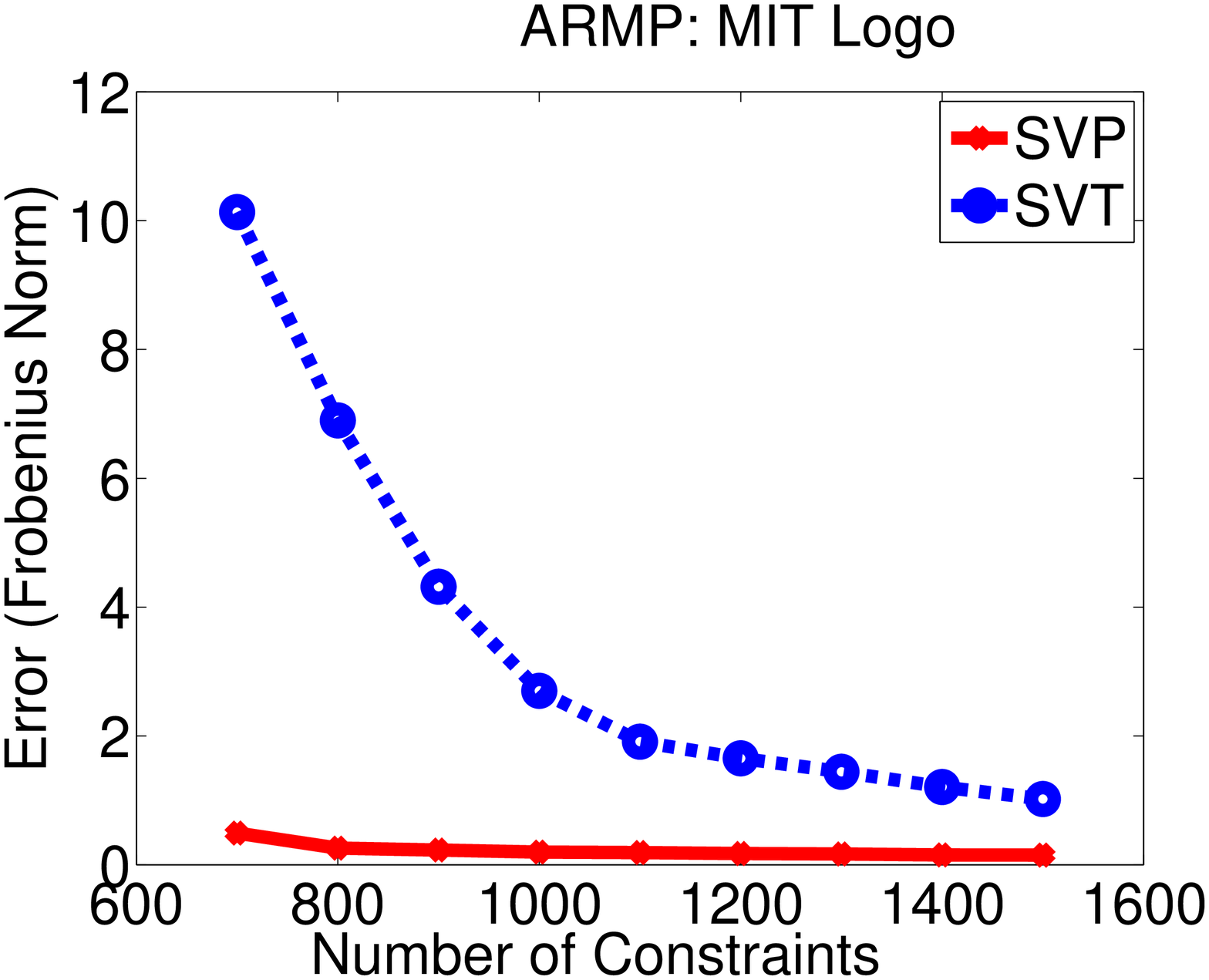} \\
(a)&(b)
     \end{tabular}

    \caption{{\bf (a)}: Time taken by $\psvd$ and SVT for random instances of Affine Rank Minimization Problem (ARMP) with optimal rank $k=5$, {\bf (b)}: Reconstruction error for the MIT logo}
\label{fig:armp}
  \end{center}

\end{figure*}

\subsection{Matrix Completion}
Next, we evaluate our method against various matrix completion methods for random low-rank matrices and uniform samples. We generate a random rank $k$ matrix $X\in\mathbb{R}^{n\times n}$ and generate random Bernoulli samples with probability $p$. Figure~\ref{fig:mctime} compares the time required by various methods (in $\log$-scale) to obtain a root mean square error (RMSE) of $10^{-2}$ for fixed $k=2$. Clearly, our method is substantially faster than the other methods.  Next, we evaluate our method for increasing $k$. Figure~\ref{fig:mctime1} compares the time required by various methods to obtain a root mean square error (RMSE) of $10^{-2}$ for fixed $n=1000$ and increasing $k$. Note that our algorithm scales well with increasing $k$ and is much faster than the other methods. 
\begin{figure*} [ht!]
\begin{center}
     \includegraphics[width=0.6\textwidth, height=5cm]{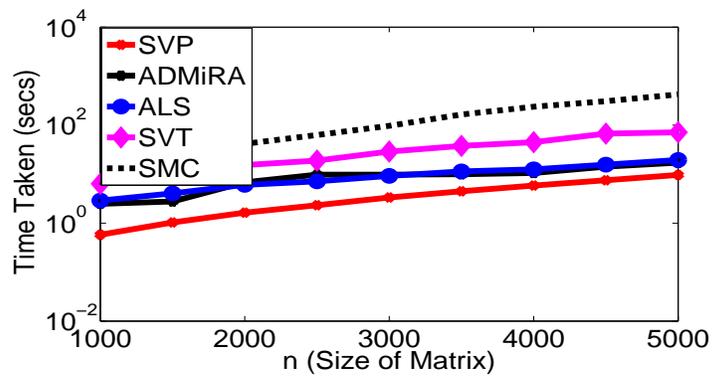}

    \caption{Running time (on log scale) for various methods  for matrix completion problem with sampling density $p=.1$ and optimal rank $k=2$.}
\label{fig:mctime}
  \end{center}

\end{figure*}
\begin{figure*} [ht!]
\begin{center}
     \includegraphics[width=0.6\textwidth, height=5cm]{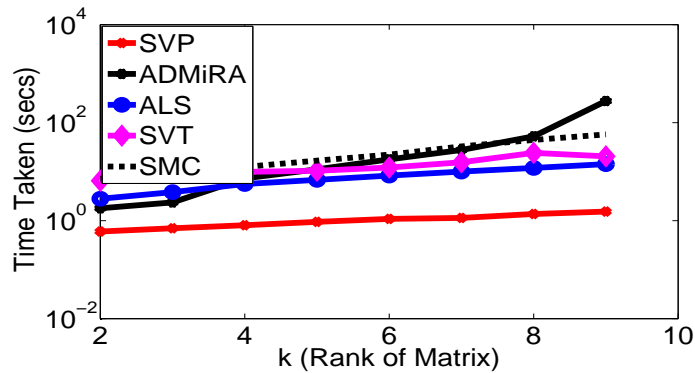}

    \caption{Running time (on log scale) for various methods  for matrix completion problem with sampling density $p=.1$ and  $n=1000$.}
\label{fig:mctime1}
  \end{center}

\end{figure*}

Finally, we study the behavior of our method in presence of noise. For this experiment, we generate random matrices of different size and add approximately $5\%$ Gaussian noise. Figure~\ref{fig:matcomp} plots error incurred and time required by various methods as $n$ increases from $1000$ to $5000$. Note that SVT is particularly sensitive to noise and incurs high RMSE.\\ 

\begin{figure*} [t!]
\begin{center}
\begin{tabular}[c]{@{\hspace{-0.03\textwidth}}c@{\hspace{-0.015\textwidth}}c@{\hspace{-0.015\textwidth}}c@{\hspace{-0.025\textwidth}}c@{}}
     \includegraphics[width=0.45\textwidth, height=4cm]{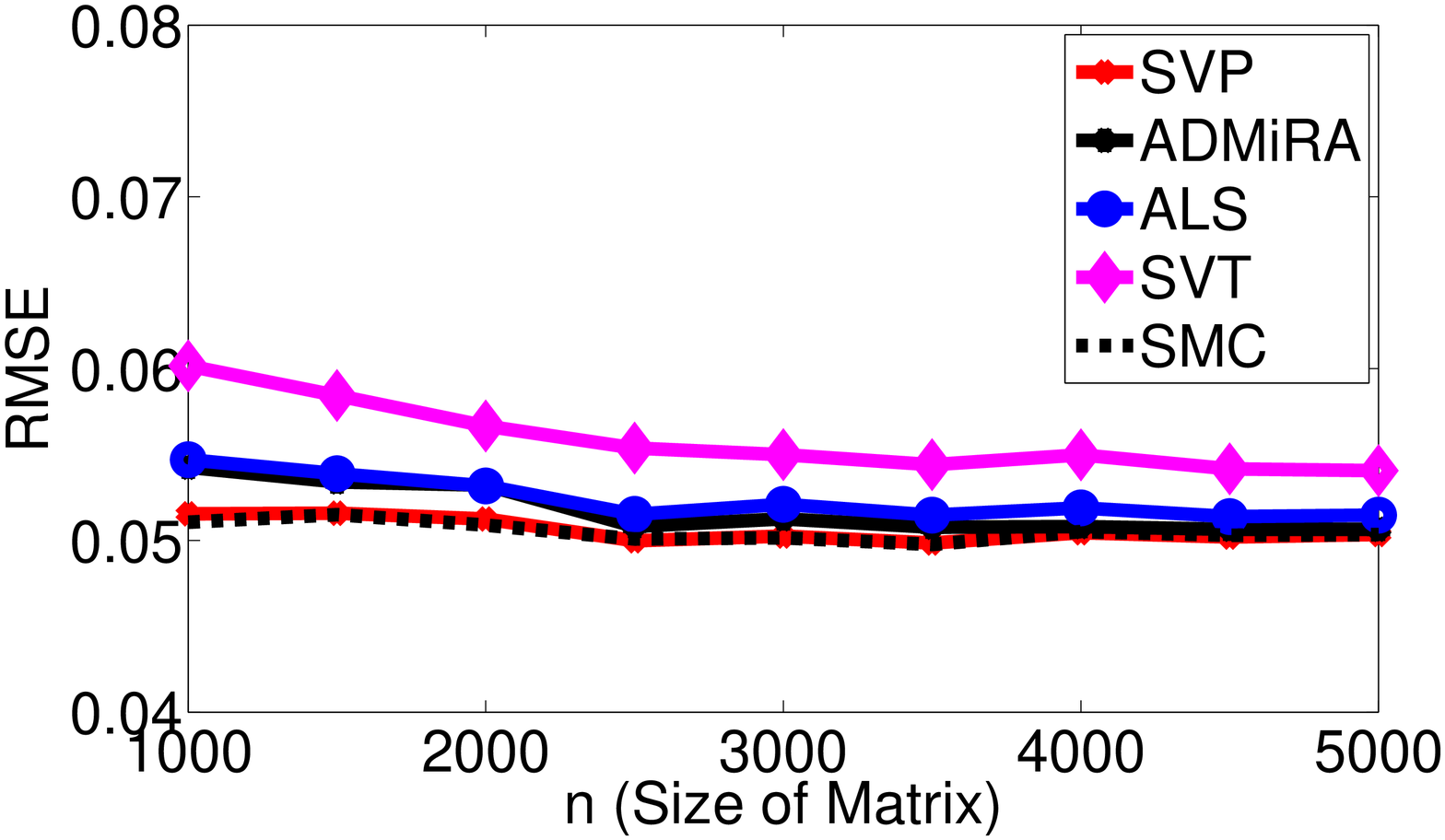}&
      \includegraphics[width=0.45\textwidth, height=4cm]{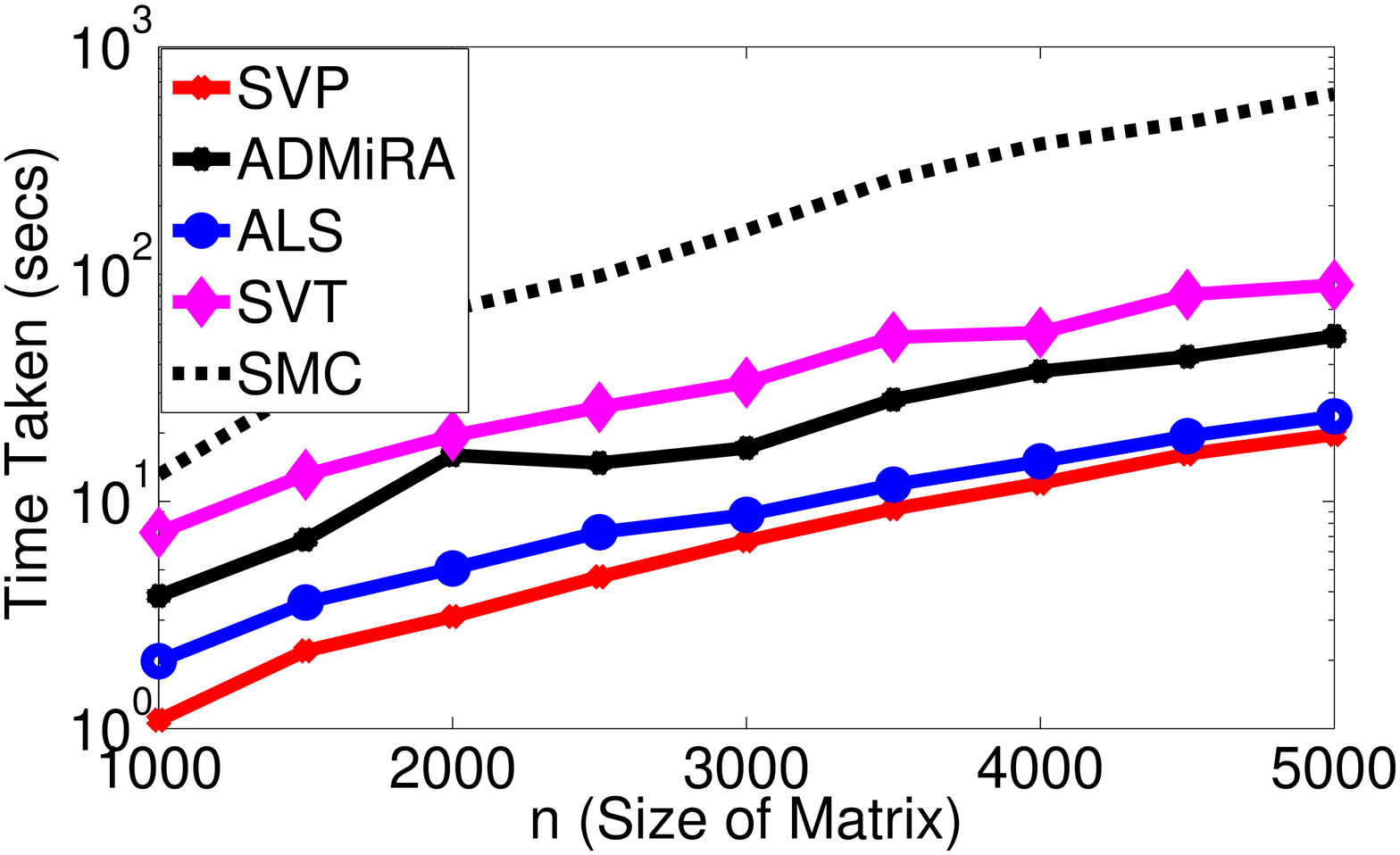} 
     \end{tabular}
   \caption{RMSE and time required by various methods for matrix completion with $p=.1$, $k=2$ and around $10\%$ of the known entries are corrupted. Note that in terms of RMSE values, $\psvd$, ALS and SMC perform about the same.}
\label{fig:matcomp}
  \end{center}
\end{figure*}

{\bf Matrix Completion: Movie-Lens Dataset}\\
Finally, we evaluate our method on the Movie-Lens dataset \cite{Movielens}, which contains 1 million ratings for $3900$ movies by $6040$ users. For $\psvd$ and ALS, we fix the rank of the matrix to be $k=15$. For $\psvd$, we set the step size $\eta_t$ to be $5/ \sqrt{t}$. $\psvd$ incurs RMSE of $1.01$ in $64.85$ seconds, while SVT incurs RMSE of $1.21$ in $1214.78$ seconds. In contrast, ALS achieves RMSE of $0.90$ in $195.34$ seconds. We attribute the relatively poor performance of $\psvd$ and SVT as compared with ALS to the fact that the ratings matrix is not sampled uniformly, thus violating a crucial assumption of both our method and SVT. Similar to Figure~\ref{fig:matcomp} (b), SVT converges much slower than SVP on the Movie-Lens data.
